\documentclass[10pt]{article}

\usepackage[margin=1in]{geometry}
\usepackage{parskip}
\usepackage[utf8]{inputenc}
\usepackage[T1]{fontenc}
\usepackage{hyperref}
\usepackage{url}
\usepackage{booktabs}
\usepackage{amsfonts}
\usepackage{amsmath}
\usepackage{dsfont}
\usepackage{amsmath}
\usepackage{amsthm}
\usepackage{amssymb}
\usepackage{enumerate}
\usepackage{subcaption}
\usepackage{bm}
\usepackage{graphicx}
\usepackage{nicefrac}
\usepackage{microtype}
\usepackage{algorithm}
\usepackage[noend]{algorithmic}
\usepackage{cleveref}
\usepackage{natbib}
\usepackage[table]{xcolor}
\usepackage{tikz}
\renewcommand{\epsilon}{\varepsilon}

\newtheorem{lemma}{Lemma}
\newtheorem{theorem}{Theorem}
\newtheorem{corollary}{Corollary}
\newtheorem{definition}{Definition}
\newtheorem{proposition}{Proposition}
\newtheorem{remark}[theorem]{Remark}
\crefname{remark}{Remark}{Remarks}

\crefname{observation}{Observation}{Observations}
\newtheorem{assumption}{Assumption}
\newenvironment{ack}{\section*{Acknowledgments}}{}

\newcommand{\E}{\mathop{\mathbb{E}}}
\newcommand{\mP}{\mathop{\mathbb{P}}}

\newcommand{\bB}{\mathbb{B}}

\newcommand{\cV}{\mathcal{V}}

\newtheoremstyle{TheoremNum}
{\topsep}{\topsep}              
{\itshape}                      
{}                              
{\bfseries}                     
{.}                             
{ }                             
{\thmname{#1}\thmnote{ \bfseries #3}}
\theoremstyle{TheoremNum}

\newcommand{\eps}{\epsilon}

\newcommand{\R}{\mathbb{R}}

\newcommand{\cZ}{\mathcal{Z}}

\newcommand{\WW}{\mathcal{W}}
\newcommand{\alg}{\mathcal{A}}
\newcommand{\talg}{\tilde{\alg}}
\newcommand{\Vol}{\text{Vol}}

\newcommand{\Unif}{\mathrm{Unif}}
\newcommand{\argmin}{\text{argmin}}
\newcommand{\ZZ}{\mathcal{Z}}
\renewcommand{\varnothing}{\emptyset}

\title{Near-Optimal Machine Unlearning Utility for Smooth Strongly Convex Losses}

\author{
    Matthew Regehr\thanks{University of Waterloo. \texttt{matt.regehr@uwaterloo.ca}.} \and
    Gautam Kamath\thanks{University of Waterloo and Vector Institute. \texttt{g@csail.mit.edu}.} \and
    Andrew Lowy\thanks{CISPA Helmholtz Center for Information Security. \texttt{lowy@cispa.de}.} 
}

\begin{document}

\maketitle

\begin{abstract}
    Machine unlearning is motivated by legal and user-facing requirements to remove the influence of individuals' data from trained models, such as the right to be forgotten. Prior work has developed algorithms and error bounds for unlearning in smooth strongly convex stochastic optimization but the fundamental statistical cost of unlearning has remained unclear. We nearly resolve this problem by proving upper and lower bounds on the excess population risk of approximate $(\varepsilon, \delta)$-unlearning; our bounds are tight up to a condition-number factor. For mean estimation over the unit ball, our upper and lower bounds match. In fact, our algorithm achieves $\eps$-unlearning, which implies a notable separation between differential privacy and unlearning: $(\eps, \delta)$-unlearning has no statistical advantage over pure $\eps$-unlearning.

    The optimal rate is the usual sampling error plus an unlearning penalty that interpolates between the retraining from scratch rate and an exponentially smaller term as $\varepsilon/d$ grows, where $d$ is the dimension of the model. The retraining penalty dominates the sampling error for large unlearning requests. In particular, retraining from scratch is information theoretically optimal up to $\eps \lesssim d$. On the other hand, for $\varepsilon \gg d$ and large unlearning requests, our $\eps$-unlearning algorithm offers an exponential accuracy improvement over retraining the model from scratch and differentially private baselines.
\end{abstract}

\section{Introduction}

Machine unlearning~\cite{cao2015towards} is motivated by legal, institutional, and user-facing requirements to remove the influence of individuals' data from trained models, such as the right to be forgotten~\cite{gdpr2016,guo2020certified}. Given a model trained on a dataset, an unlearning procedure receives a request to ``delete'' or ``unlearn'' a subset of the training samples and must update the model so that the result is statistically close to what would have been produced had those samples never been used.  

We consider unlearning in stochastic convex optimization (SCO). Given i.i.d.\ samples
\(Z=(z_1,\ldots,z_n)\sim P^n\), the goal is to approximately minimize the population loss function
\begin{equation}
    \label{eq:sco}
    \min_{w\in\WW} \left\{ F_P(w) := \E_{z\sim P}[f(w,z)] \right\},
\end{equation}
where \(\WW\subseteq\R^d\) is a convex parameter domain and \(f(\cdot,z)\) is a loss function. The quality of a solution \(w\) is measured by its \emph{excess population risk}
$\Delta F_P(w) := F_P(w) - F_P^*$, where~$F_P^* := \min_{w'\in\WW} F_P(w').$ In the unlearning setting, a learning algorithm \(\alg\) first receives the full dataset \(Z\) and outputs a model, possibly together with side information. Later, an unlearning algorithm \(\talg\) receives an unlearning request \(U\subseteq Z\), with \(|U|\le m\), and must output an updated model. Informally, approximate unlearning requires that this updated model be statistically difficult to distinguish (in the max-divergence sense, \`a la differential privacy) from the output of the learning algorithm run directly on the retained dataset \(Z\setminus U\).

There are two na\"ive baseline methods for unlearning. First, \textit{differential privacy} (DP)~\cite{dwork2006calibrating} automatically gives unlearning: if the original training algorithm is private enough to hide the contribution of any possible unlearning set, then no special update is needed at unlearning time. However, this can be overly conservative, since a DP algorithm must hide all possible changes in advance, before the unlearning set is known; it does not leverage knowledge of $U$. Second, discarding the trained model and \textit{retraining the model from scratch} on \(Z\setminus U\) gives exact \(0\)-unlearning; but it discards part of the dataset and intuitively seems inefficient from a utility perspective. Can we improve over these two baseline approaches?



In this paper, we investigate the \textit{smallest excess population risk that is achievable} in smooth strongly convex SCO subject to this approximate unlearning constraint. While considerable attention in the SCO unlearning literature has been devoted to computational and storage considerations, it is difficult to cleanly track progress on that front due to widely varying utility rates. It is especially difficult to distinguish inherent error from errors arising from a particular computational approach. Therefore, establishing the minimax rate for SCO unlearning provides a clear target and baseline from which to compare algorithms on computational terms. This mirrors a long line of research on private SCO in which \cite{bft19} first establish the optimal rate for private SCO but with suboptimal time complexity. Subsequent work builds upon this foundation to achieve the optimal rate by faster algorithms and in increasing generality \cite{FeldmanKoTa20,bfgt20,AsiFeKoTa21,AsiLeDu21}.

A recent line of work develops algorithms and excess-risk upper bounds for unlearning in smooth strongly convex stochastic optimization~\cite{sekhari2021remember,youseffutility,van2025forget,zou2025certified,qiao2025hessian}.
These works initiate the study of utility, computation, and storage tradeoffs under different assumptions and show that unlearning can improve over differentially private baselines in certain regimes. However, no excess risk bounds were given for the retraining from scratch (RFS) baseline in prior work and it was unclear whether RFS is improvable. In particular, the following fundamental question has remained open:

\begin{center}
    \fbox{
    \parbox{0.78\linewidth}{
    \textbf{Question.}
    What is the minimax optimal excess population risk for \((\varepsilon, \delta)\)-unlearning up to \(m\) samples in smooth strongly convex stochastic optimization?
    }}
\end{center}

\paragraph{Our contributions.}
We resolve the question above up to a condition-number factor. Our upper and lower bounds characterize the minimax optimal excess-risk for smooth strongly convex stochastic optimization as
\begin{equation}
    \label{eq:informal-rate}
    \frac{1}{n}
    +
    \left(\frac{m}{n}\right)^2
    e^{-2\varepsilon/(d+2)}.
\end{equation}

The rate~\eqref{eq:informal-rate} has a natural interpretation and yields structural insights that may be relevant to the broader unlearning program. The \(1/n\) term is the usual sampling error present even without unlearning requirements~\cite{ny83}, while the second term is the unlearning penalty. Our analysis shows that RFS achieves unlearning penalty \((m/n)^2\), which improves over the standard DP baseline. It is immediate that this is statistically optimal for exact $0$-unlearning and our lower bound shows that this is unimprovable for \(\varepsilon\lesssim d\). That is, RFS is information theoretically optimal up to relatively weak unlearning constraints. On the other hand, for a weaker unlearning constraint \(\varepsilon \gg d\), our algorithm can leverage information from the unlearned data to reduce the unlearning penalty by the exponential factor \(\exp(-2\varepsilon/(d+2))\). Thus, for large unlearning requests $m \gg \sqrt{n}$ and a weak unlearning constraint $\varepsilon \gg d$, information theoretically optimal unlearning is exponentially more accurate than RFS, DP-based unlearning, and prior unlearning algorithms. These comparisons are summarized in \Cref{tab:informal-comparison}, which reports the closest apples-to-apples baselines for our rate comparison; broader related work is discussed in Appendix~\ref{app:related-work}.

Our main contributions are:
\begin{enumerate}
    \item \textbf{Nearly tight minimax bounds for smooth strongly convex SCO.}
    We give an \((\varepsilon, 0)\)-unlearning algorithm for smooth strongly convex stochastic optimization whose excess risk matches~\eqref{eq:informal-rate} up to a condition-number factor. We complement this with a lower bound showing that no unlearning algorithm can improve the dependence on \(n,m,\varepsilon, \delta,d\).

    \item \textbf{A sharp characterization for mean estimation.}
    For mean estimation over the unit ball in \(\R^d\), our upper and lower bounds match up to constants.

    \item \textbf{Sharper analyses of baseline and prior algorithms.}
    We give a sharp analysis of the exact \(0\)-unlearning baseline retrain-from-scratch, showing that it achieves \(1/n+(m/n)^2\) excess risk and is optimal when \(\varepsilon\lesssim d\). In the appendix, we sharpen the analysis of the Newton-step algorithm of~\cite{sekhari2021remember} to get a quadratic improvement in the unlearning error term. We also identify a bug in the utility analysis of \cite{youseffutility} and give a corrected guarantee for their warm-start ERM procedure.
\end{enumerate}

Our lower bounds apply to all \((\varepsilon, \delta)\)-unlearning algorithms, unlike the DP-based lower bounds of \cite{huang2023tight}. In particular, allowing $\delta > 0$ offers no meaningful utility improvements for unlearning. This is in stark contrast to differential privacy, where many statistical tasks exhibit a significant separation between the $\delta = 0$ and $\delta > 0$ settings (see e.g. \cite{su16}).

\begin{table}[t]
    \centering
    \caption{Excess-risk rates for \(\varepsilon\)-unlearning up to \(m\) points. We suppress logarithmic, Lipschitz, strong-convexity, smoothness, and condition-number factors. For prior algorithms, we report the sharpened or corrected rates proved in the Appendix.}
    \label{tab:informal-comparison}
    \footnotesize
    \begin{tabular}{p{0.31\linewidth}p{0.34\linewidth}p{0.25\linewidth}}
        \toprule
        Method / result & Excess-risk rate & Reference / comments \\
        \midrule

        DP baseline &
        \(\displaystyle
        \frac{1}{n}
        +
        \frac{d^2m^2}{\varepsilon^2 n^2}
        \) &
        Group DP and~\cite{AsiLeDu21} \\[0.8em]

        Retrain from scratch (RFS) &
        \(\displaystyle
        \frac{1}{n}
        +
        \left(\frac{m}{n}\right)^2
        \) &
        Exact \(0\)-unlearning; \Cref{thm:rfs} \\[0.8em]

        ERM / warm-start unlearning &
        \(\displaystyle
        \frac{1}{n}
        +
        \left(\frac{m}{n}\right)^2
        \) &
        \cite{youseffutility}; corrected in \Cref{thm:youseff-corrected}; approximate unlearning; logarithmic speedup over RFS \\[0.8em]

        Newton-step unlearning &
        \(\displaystyle
        \frac{1}{n}
        +
        \left(\frac{m}{n}\right)^2
        +
        \frac{d^2m^4}{\varepsilon^2 n^4}
        \) &
        \cite{sekhari2021remember}; sharpened in \Cref{thm:saks-improved}; requires Lipschitz Hessian \\[0.9em]



        \rowcolor{yellow!18}
        Our upper/lower bound &
        \(\displaystyle
        \frac{1}{n}
        +
        \left(\frac{m}{n}\right)^2
        e^{-2\varepsilon/(d+2)}
        \) &
        Optimal up to condition number; tight for mean estimation  \\
        \bottomrule
    \end{tabular}
\end{table}

\paragraph{Techniques.}
Our key algorithmic idea is a novel way to exploit information available at unlearning time. The empty-deletion run on a retained dataset \(S = Z\setminus U\) returns a point near \(\hat w_S\) with high probability, but also places a small amount of probability mass over a region covering all possible full-data ERM solutions. When an unlearning request \(U\) arrives, the unlearner forms \(S=Z\setminus U\), reconstructs this noise distribution by retraining from scratch.
It then moves the high-probability mass near the reduced-data ERM solution \(\hat w_{Z\setminus U}\) to the full-data ERM solution \(\hat w_Z\).
We calibrate the optimal geometry and probabilities needed to satisfy the \(\varepsilon\)-likelihood-ratio constraint while achieving an advantage over retraining from scratch by a factor of \(e^{-\Theta(\varepsilon/d)}\).

Our lower bound uses packing techniques, but departs from standard DP packing arguments. In unlearning, the algorithm may see the unlearning set and arbitrary side information, so DP-style neighboring-dataset indistinguishability alone is not enough. Instead, we construct many separated mean-estimation instances whose nonzero signal samples can be deleted with high probability. After deletion, the retained datasets collapse to a common reference dataset, so \(\varepsilon\)-unlearning forces the output distributions for all packed instances to be close to one common reference distribution. A packing argument then yields the matching \(e^{-\Theta(\varepsilon/d)}\) unlearning penalty.

\subsection{Preliminaries}

Let \(\|\cdot\|\) denote the Euclidean norm. Let $\bB_d$ denote the $d$-dimensional Euclidean unit ball and let \(\WW\subseteq\R^d\) be a domain. For a differentiable function \(h:\WW\to\R\), we say \(h\) is \(L\)-\textit{Lipschitz} if \(|h(w)-h(u)|\le L\|w-u\|\) for all \(w,u\in\WW\); \(\mu\)-\textit{strongly convex} if \(h(u)\ge h(w)+\langle \nabla h(w),u-w\rangle+\frac{\mu}{2}\|u-w\|^2\) for all \(w,u\in\WW\); and \(\beta\)-\textit{smooth} if \(h(u)\le h(w)+\langle \nabla h(w),u-w\rangle+\frac{\beta}{2}\|u-w\|^2\) for all \(w,u\in\WW\). For smooth strongly convex losses, \(\kappa:=\beta/\mu\) denotes the \textit{condition number}.

We make the following assumption throughout, which is standard in the theoretical study of unlearning~\cite{sekhari2021remember,youseffutility}. 
Unlike~\cite{sekhari2021remember}, we do not require that Hessian of the loss is Lipschitz.

\begin{assumption}[Smooth strongly convex SCO]
    \label{ass:smooth-sc-sco}
    The following conditions hold:
    \begin{enumerate}
        \item The parameter domain \(\WW\subseteq\R^d\) is closed, convex, and has finite diameter \(D = \sup_{w,w'\in\WW}\|w-w'\|\).
        \item For every \(z\in\ZZ\), the loss \(f(\cdot,z):\WW\to\R\) is differentiable and \(L\)-Lipschitz on \(\WW\).
        \item For every \(z\in\ZZ\), the loss \(f(\cdot,z)\) is \(\mu\)-strongly convex on \(\WW\).
        \item For every \(z\in\ZZ\), the loss \(f(\cdot,z)\) is \(\beta\)-smooth on \(\WW\).
        \item The population minimizer $w_P^* := \argmin_{w \in \WW} F_P(w)$ satisfies $\nabla F_P(w_P^*)=0$.
    \end{enumerate}
\end{assumption}

Moreover, we will fix throughout a dataset size $n \geq 1$ as well as unlearning capacity $1 \leq m < n$. As in \cite{sekhari2021remember}, non-trivial unlearning guarantees can only be provided in the regime where $m \leq cn$ for a constant $c < 1$. In particular, we assume throughout that $n - m = \Omega(n)$.
Now, for a dataset \(Z=(z_1,\ldots,z_n)\), let \(\widehat F_Z(w):=\frac1n\sum_{i=1}^n f(w,z_i)\) and, for an unlearning set \(U\subseteq Z\), let \(Z\setminus U\) denote the retained dataset.

\paragraph{Approximate Unlearning.} For random variables \(X,Y\) on the same measurable space, write \(X\approx_{\varepsilon,\delta}Y\) if for every measurable event \(E\), both \(\Pr(X\in E)\le e^\varepsilon\Pr(Y\in E)+\delta\) and \(\Pr(Y\in E)\le e^\varepsilon\Pr(X\in E)+\delta\). We follow the definition of approximate unlearning used in~\cite{sekhari2021remember} but do not impose any storage constraints on the unlearner. On input data $Z$, any learning algorithm $\alg(Z)$ is applied. Afterward, the unlearner $\talg$ takes as inputs the unlearning set $U$ and the dataset $Z$ and must return a model that is $\approx_{\varepsilon,\delta}$-close to the model that the unlearner would return given a learner input $Z \setminus U$ in the absence of any unlearning request. In general, the unlearner would also take the output of the learner $\alg(Z)$ as well as some stored information $T(Z)$ instead of necessarily the whole dataset $Z$. In our case, we allow the unlearner full storage, meaning $T(Z) = Z$. Thus, in our setting, the unlearner has access to the full data set $Z$ and the unlearning set $U$, as well as the learned model $\alg(Z)$.

\begin{definition}[Approximate \(\varepsilon\)-unlearning]
    \label{def:eps-unlearning}
    An algorithm \(\talg\) satisfies \((\varepsilon, \delta)\)-unlearning for up to \(m\) deletions if, for every dataset \(Z \in \ZZ^n\) and every deletion set \(U\subseteq Z\) with \(|U|\le m\),
    \[
        \talg(U, Z)
        \approx_{\varepsilon,\delta}
        \talg(\varnothing,Z \setminus U).
    \]
    If $\delta = 0$, we say that $\talg$ satisfies $\eps$-unlearning.
\end{definition}

We measure utility by \textit{worst-case expected excess population risk after up to \(m\) deletions}.

\begin{definition}
    \label{def:unlearning_utility}
    We say that $\talg$ achieves expected excess unlearning risk $\alpha$ if for every distribution $P \in \Delta(\ZZ)$ and $n - m \leq N \leq n$ and any adversarial unlearning request $U(Z) \subseteq Z$ such that $|Z \setminus U(Z)| \geq n - m$, we have
    \[
        \E_{\talg, Z \sim P^N}\left[ F_P(\talg(U(Z), Z)) - F_P^* \right] \leq \alpha.
    \]
\end{definition}

We note that our definition is slightly different than that of prior work including \cite{sekhari2021remember}, which only requires good performance on $\talg(U, Z), Z \sim P^n$, the model released after an unlearning request. Our stronger definition also requires good performance on $\talg(\varnothing, S), S \sim P^{n - m}$, the model trained on the dataset in which the unlearned data was not included to begin with. We present in \Cref{app:post_unlearning} analogous matching upper and lower bounds for the weaker definition more common in prior work. We argue that the optimal unlearner in that setting behaves in a way contrary to the spirit of machine unlearning and we therefore advocate the adoption of our stronger notion of utility.

\textit{Retraining from scratch (RFS)} is defined by $\talg(U, Z) := \argmin_{w \in \WW} \widehat{F}_{Z \setminus U}(w)$ in combination with empty unlearning request run $\talg(\varnothing, Z \setminus U) := \argmin_{w \in \WW} \widehat{F}_{Z\setminus U}(w)$; this satisfies exact \(0\)-unlearning by definition. 

Differential privacy gives another generic route: if \(\alg\) is $\eps$-DP for groups of size \(m\), then \(\alg(Z)\) is already $\eps$-close in distribution to \(\alg(Z\setminus U)\), so no unlearning-time update is needed (i.e., $\talg(U, Z) = \alg(Z)$). In particular, by group privacy, an \((\varepsilon/m)\)-DP algorithm at the individual-sample level yields \(\varepsilon\)-unlearning for unlearning sets of size at most \(m\).

\section{Upper Bounds for $\eps$-Unlearning}

In this section, we provide a novel algorithm for $\varepsilon$-unlearning that achieves excess population risk
\begin{align*}
    O\left(\kappa \frac{L^2}{\mu}\left(\frac{1}{n} + \left( \frac{m}{n} \right)^2 e^{-2\eps/(d+2)} \right)\right),
\end{align*}
for SCO and $O(1/n + (m/n)^2 e^{-2\eps/(d+2)})$ mean squared error for mean estimation on a unit ball. Thus, our algorithm gives an exponential improvement over the retrain-from-scratch penalty of $\Theta((m/n)^2)$ in the regime $\varepsilon = \Omega(d)$. These upper bounds are tight up to $O(\kappa)$ and $O(1)$ respectively. 

For any dataset \(S\), write $\widehat F_S(w):=\frac1{|S|}\sum_{z\in S} f(w,z)$ and $\hat w_S\in\argmin_{w\in\WW}\widehat F_S(w).$

\paragraph{Our algorithm.}
Consider the following two extreme algorithms for unlearning. To achieve high accuracy at the expense of providing no unlearning guarantees, an ERM unlearner given a dataset $Z$ and an unlearning request $U$ should simply ignore the request and return an approximate full-data ERM solution $\tilde w_Z$. On the other hand, to achieve a perfect unlearning guarantee at the expense of accuracy, the unlearner should retrain from scratch and return an approximate ERM solution on the reduced dataset $\tilde w_{Z \setminus U}$.

At a high-level, our ERM unlearner works by optimally interpolating between these two algorithms to provide the desired level of unlearning while maximizing the accuracy. More precisely, given a dataset $Z$ and a (possibly empty) unlearning request $U$, our unlearner returns with high probability a solution sampled uniformly from a small ball centered around the full-data ERM solution $\tilde w_Z$. However, when the learner is run ``dry'' on a reduced dataset $S$ and an empty unlearning request, the unlearner must provide plausible deniability to any true runs of the unlearner provided with a dataset $Z$ and an unlearning request $U$ that resolves to $S = Z \setminus U$. To achieve this, the unlearner also returns, with low probability, a sample from a wider ball centered around the reduced-data ERM $\tilde w_S$ that is just large enough to capture any full-data ERM solutions $\tilde w_Z$ that may reduce to $S = Z \setminus U$ by a legal unlearning request. We note that, critically, the approximate ERM solution $\tilde{w}_S$ is deterministic in the sense that it only depends on $Z$ and $U$ through $Z \setminus U$. \Cref{fig:unlearner} visualizes the sampling distributions of the unlearning process when run dry compared to a true run of the unlearner. \Cref{alg:core-swap} provides the pseudocode, which is written for a generic input dataset \(T\). In the true unlearning call, \(T = Z\) and the deletion set is \(U\). In the reference empty-deletion dry run, \(T = Z\setminus U\) and the deletion set is \(\varnothing\).

\begin{figure}[htbp]
    \centering
    \includegraphics[width=0.85\linewidth]{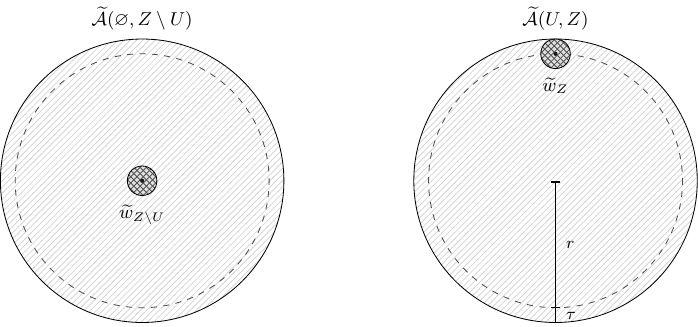}
    \caption{We compare the distribution of $\talg(\varnothing, Z \setminus U)$ to $\talg(U, Z)$. Lightly shaded regions are sampled uniformly with low probability and heavily shaded regions are sampled uniformly with high probability. With high probability, $\talg(\varnothing, Z \setminus U)$ returns a solution very close to the RFS ERM solution $\tilde{w}_{Z \setminus U}$ and, with small probability, a sample from a ball surrounding it. $\talg(U, Z)$ has the same distribution except that ball centered around the inaccurate RFS solution $\tilde{w}_{Z \setminus U}$ is shifted to capture the full ERM solution $\tilde{w}_Z$.}
    \label{fig:unlearner}
\end{figure}

\begin{algorithm}[t]
    \caption{Core-swap \(\varepsilon\)-unlearning for ERM}
    \label{alg:core-swap}
    \begin{algorithmic}[1]
        \REQUIRE Dataset \(T\); deletion set \(U\subseteq T\); parameters \(L,\mu,n,m,d,\varepsilon\)
        \STATE Set \(S\gets T\setminus U\), \(\rho \gets \frac{L}{\mu \sqrt{n}}\), \(r\gets \frac{2Lm}{\mu(n - m)} + 2\rho\), \(\tau\gets r\min\{1,2e^{-\varepsilon/(d + 2)}\}\), and \(\eta\gets \frac{1}{(e^\varepsilon - 1)\left( \frac{\tau}{\tau + r} \right)^d + 1}\)
        \STATE Compute deterministic $\tilde w_T$ and $\tilde w_S$ such that $\lVert \tilde w_T - \hat w_T \rVert, \lVert \tilde w_S - \hat w_S \rVert \leq \rho$
        \STATE Sample $w \sim
            \begin{cases}
                \Unif((\tau + r)\bB_d + \tilde w_S) & \text{w.p. } \eta \\
                \Unif(\tau\bB_d + \tilde w_T) & \text{w.p. } 1 - \eta
            \end{cases}
        $
        \RETURN $\argmin_{w' \in \WW} \lVert w' - w \rVert$
    \end{algorithmic}
\end{algorithm}

\paragraph{Guarantees of \Cref{alg:core-swap}.}
We now formally record our main upper bound guarantees. Note that by gradient query we mean a single evaluation of $\nabla_w f(w, z)$ and by projection query we mean a Euclidean projection onto the convex set $\WW$. In particular, under the assumption that gradient and projection queries can be evaluated in $O(d)$ time, the overall runtime of our algorithm is $\tilde O(\kappa \cdot nd)$.

\begin{theorem}[Main Upper Bound]
    \label{thm:core-swap-upper}
    Grant \Cref{ass:smooth-sc-sco}. Then the unlearner \Cref{alg:core-swap} satisfies
    \(\varepsilon\)-unlearning, requires $O(d)$ time for randomization as well as $\tilde O(\kappa \cdot n)$ gradient and projection queries. Moreover, \Cref{alg:core-swap} achieves expected excess unlearning risk
    \[
        \alpha
            =
            O\left(
            \kappa\frac{L^2}{\mu}
            \left(
                \frac1n+
                \left(\frac mn\right)^2e^{-2\eps/(d+2)}
            \right)
            \right).
    \]
    In the special case of mean estimation over the unit ball, i.e.
    \(\WW=\ZZ=\bB_d\) and \(f(w,z)=\frac12\|w-z\|^2\), \Cref{alg:core-swap} achieves expected excess unlearning risk
    \[
        \alpha = O\left(
            \frac1n+
            \left(\frac mn\right)^2e^{-2\eps/(d+2)}
        \right).
    \]
\end{theorem}

The key proof idea is that the empty-deletion distribution for \(S=Z\setminus U\) places small probability mass near every possible full-data ERM \(\tilde w_Z\). The unlearning step moves the high-probability mass from a region near \(\tilde w_S\) onto a region near \(\tilde w_Z\). Verifying the unlearning guarantee thus involves bounding just one likelihood ratio, while utility follows from carefully controlling the slack radius $\tau$, the probability $\eta$ of selecting a noisy solution, as well as the ERM approximation error. Our utility analysis separates the statistical error of \(\tilde w_Z\) from the randomization cost of unlearning.

To formally prove \Cref{thm:core-swap-upper}, we proceed in three steps. First, \Cref{prop:core-swap-unlearning} shows
that \Cref{alg:core-swap} satisfies \(\varepsilon\)-unlearning. Second, \Cref{prop:core-swap-erm-distance} shows that its output is close to the approximate full-data ERM \(\tilde w_Z\). Finally, we combine this ERM-distance guarantee with the standard stability-induced distance bound between
\(\hat w_Z\) and the population minimizer \(w_P^*\) to obtain our excess population risk guarantee.

\begin{proposition}
    \label{prop:core-swap-unlearning}
    The unlearner \Cref{alg:core-swap}, satisfies
    \(\varepsilon\)-unlearning.
\end{proposition}

\begin{proof}
    Unlearning is preserved by postprocessing, so we may exclude the final projection onto $\WW$ from our analysis.
    Now, fix $U \subseteq Z \in \ZZ^N$ with $n - m \leq N \leq n$ and $|Z \setminus U| \geq n - m$. The distribution of $\talg(\varnothing, Z \setminus U)$ before projection is
    \[
        W_{\varnothing, Z \setminus U} = (1 - \eta) \Unif(\underbrace{\tau \bB_d + \tilde{w}_{Z \setminus U}}_{=: G}) + \eta \Unif(\underbrace{(\tau + r) \bB_d + \tilde{w}_{Z \setminus U}}_{=: B})
    \]
    whereas the distribution of running $\talg(U, Z)$ before projection is 
    \[
        W_{U, Z} = (1 - \eta) \Unif(\underbrace{\tau \bB_d + \tilde{w}_Z}_{=: G'}) + \eta \Unif(\underbrace{(\tau + r) \bB_d + \tilde{w}_{Z \setminus U}}_{= B}).
    \]

    We just need to argue that the likelihood ratio between these distributions lies in $[e^{-\eps}, e^\eps]$. Indeed, by the standard stability bound for strongly convex ERM (c.f. \cite[Lemma 6]{sekhari2021remember}),
    \[
        \lVert \tilde w_Z - \tilde w_{Z \setminus U} \rVert
            \leq \lVert \hat w_Z - \hat w_{Z \setminus U} \rVert + 2\rho
            \leq \frac{2L|U|}{\mu |Z \setminus U|} + 2\rho
            \leq \frac{2Lm}{\mu (n - m)} + 2\rho
            = r.
    \]
    In particular, $G'$ is contained in $B$ and hence the likelihood ratio between these distributions differs only for $w$ contained in either $G$ or $G'$, exclusively. In the first case, the ratio of probability density functions is 
    \begin{align*}
        \frac{dW_{\varnothing, Z \setminus U}}{dW_{U, Z}}(w)
            = \frac{(1 - \eta)/\Vol(G) + \eta/\Vol(B)}{\eta/\Vol(B)}
            = 1 + \left( \frac{1}{\eta} - 1 \right)\frac{\Vol(B)}{\Vol(G)}
            = 1 + \left( \frac{1}{\eta} - 1 \right)\left( \frac{\tau + r}{\tau} \right)^d
            = e^\eps
    \end{align*}
    by choice of $\eta$ and, analogously, we have $\frac{dW_{\varnothing, Z \setminus U}}{dW_{U, Z}}(w) = e^{-\eps}$ for $w \in G' \setminus G$. In particular, $\tilde W_{\varnothing, Z \setminus U}$ and $\tilde W_{U, Z}$ are $(\eps, 0)$-indistinguishable, as desired.
\end{proof}

\begin{proposition}[Distance to the Full-Data ERM]
    \label{prop:core-swap-erm-distance}
    For every $n - m \leq N \leq n$, dataset \(Z \in \cZ^N\), and deletion request
    \(U \subseteq Z\) with \(|Z \setminus U| \geq n - m\), the output $\talg(U, Z)$ of \Cref{alg:core-swap} satisfies
    \[
        \E_{\talg}
        \left[
            \left\|
            \talg(U, Z)-\tilde w_Z
            \right\|^2
        \right]
        =
        O\left(
            \frac{L^2}{\mu^2}
            \left(\frac mn\right)^2
            e^{-2\eps/(d+2))} +
            \frac{L^2}{\mu^2 n}
        \right).
    \]
\end{proposition}

\begin{proof}
    Recall as in the proof of \Cref{prop:core-swap-unlearning} that $\talg(U, Z)$ is exactly the projection of
    \[
        w \sim W = (1 - \eta) \Unif(\tau \bB_d + \tilde{w}_Z) + \eta \Unif((\tau + r) \bB_d + \tilde{w}_{Z \setminus U})
    \]
    onto the convex domain $\WW$. But $\tilde{w}_Z \in \WW$, so  
    \begin{align*}
        \E_{\talg}[\| \talg(U, Z) - \tilde w_Z \|^2]
            \leq \E_{w \sim W}[\| w - \tilde w_Z \|^2]
            \leq (1 - \eta) \tau^2 + \eta \max_{w \in (\tau + r) \bB_d + \tilde{w}_{Z \setminus U}} \| w - \tilde w_Z \|^2.
    \end{align*}

    Recalling that $n - m = \Omega(n)$, the first term is bounded by
    \[
        \tau^2
            = \left(r \min\left\{1, 2e^{-\eps/(d+2)}\right\}\right)^2
            \lesssim \frac{L^2}{\mu^2} \left( \left( \frac{m}{n - m} \right)^2 + \frac{1}{n} \right) e^{-2\eps/(d+2)}
            \lesssim \frac{L^2}{\mu^2} \left( \frac{m}{n} \right)^2 e^{-2\eps/(d+2)} + \frac{L^2}{\mu^2 n}.
    \]

    As for the second term, recalling that $\|\tilde w_{Z \setminus U} - \tilde w_Z\| \leq r$ as well as our choice of $\tau \leq r$, we have
    \[
        \|w - \tilde w_Z\|^2
            \lesssim \|w - \tilde w_{Z \setminus U}\|^2 + \|\tilde w_{Z \setminus U} - \tilde w_Z\|^2
            \leq (\tau + r)^2 + 2r^2
            \lesssim r^2
            \lesssim \frac{L^2}{\mu^2} \left( \frac{m}{n} \right)^2 + \frac{L^2}{\mu^2 n}
    \]
    for any \(w\in(\tau + r) \bB_d + \tilde{w}_{Z \setminus U}\). Moreover, if $2e^{-\eps/(d+2)} > 1$, then $\eta \leq 1 = 1^2 \lesssim e^{-2\eps/(d+2)}$. More importantly, if $2e^{-\eps/(d+2)} \leq 1$, then clearly $1 \leq \frac{1}{2}e^\eps$ and $\tau = 2re^{-\eps/(d+2)}$ by choice of $\tau$, so it follows that
    \[
        \eta
            = \frac{1}{(e^\eps - 1) \left( \frac{\tau}{\tau + r} \right)^d + 1}
            \lesssim e^{-\eps} \left( \frac{2r}{\tau} \right)^d
            = e^{-\eps} \left( e^{\eps/(d+2)} \right)^d
            = e^{-\eps\left(1 - \frac{d}{d+2} \right)}
            = e^{-2\eps/(d+2)}
    \]
    in this case as well. The result now follows by combining these bounds.
\end{proof}

\begin{lemma}[ERM Distance to the Population Minimizer]
    \label{lem:erm-distance}
    For any $N \geq n - m$ and distribution \(P\) over \(\cZ\), we have
    \[
        \E_{Z\sim P^N}\left[ \|\hat w_Z-w_P^*\|^2 \right]
        =
        O\left(\frac{L^2}{\mu^2 n}\right),
    \]
    where \(w_P^* = \argmin_{w\in\WW}F_P(w)\).
\end{lemma}

\begin{proof}
    This follows from the standard expected excess-risk bound for ERM with
    \(L\)-Lipschitz, \(\mu\)-strongly convex losses,
    \[
        \E_{Z\sim P^N} \left[ F_P(\hat w_Z)-F_P^* \right]
            = O\left(\frac{L^2}{\mu N}\right)
            = O\left(\frac{L^2}{\mu n}\right),
    \]
    see e.g.~\cite{shalev2009stochastic}, together with $\mu$-strong convexity of \(F_P\) and $N = \Omega(n)$.
\end{proof}

\begin{lemma}[ERM Approximation]
    \label{lem:erm-approximation}
    For any error tolerance $\rho > 0$, there exists an algorithm that takes a dataset $Z \in \cZ^N$, makes $O(\kappa \cdot N \cdot \log(D/\rho))$ gradient and projection queries, and deterministically computes $\tilde w_Z$ such that
    $\lVert \tilde w_Z - \hat w_Z \rVert \leq \rho$.
\end{lemma}

An algorithm achieving the claimed gradient query complexity\footnote{In fact, accelerated projected gradient descent improves the \(\kappa\) dependence to \(\sqrt{\kappa}\) (see e.g., \cite[Theorem 10.42]{beck2017first}). We do not pursue such improvements here as runtime and condition number dependence are not our focus.} and error rate is projected gradient descent, which is analyzed in Theorem 3.10 of \cite{bubeck2015convex}.

We now combine the above results to prove our main upper bound theorem.

\begin{proof}[Proof of \Cref{thm:core-swap-upper}]
    \textbf{Unlearning.} The unlearning claim is exactly \Cref{prop:core-swap-unlearning}.

    \paragraph{Runtime.} Uniform sampling from $d$-dimensional balls of the form $a \bB_d + b$ can be implemented in $O(d)$ time by noticing that
    \[
        G \sim N(0, 1)^d, T \sim \Unif([0, 1]) \implies T^{1/d} \cdot \frac{G}{\|G\|} \sim \Unif(\bB_d).
    \]
    Furthermore, by choice of $\rho = \frac{L}{\mu \sqrt{n}}$ \Cref{lem:erm-approximation} ensures we may compute $\tilde{w}_T$ and $\tilde{w}_S$ with the desired precision using $O(\kappa \cdot n \cdot \log(D/\rho)) = \tilde O(\kappa \cdot n)$ gradient and projection queries.

    \paragraph{Excess risk.} Fix any unlearning request strategy $U(Z) \subseteq Z$ with $|Z \setminus U(Z)| \geq n - m$ as well as $n - m \leq N \leq n$ and consider smooth strongly convex SCO. By smoothness of
    \(F_P\) and the assumption \(\nabla F_P(w_P^*)=0\),
    \[
        F_P(w)-F_P^*
        \le
        \frac{\beta}{2}\|w-w_P^*\|^2.
    \]
    Applying \(\|a+b\|^2\le 2\|a\|^2+2\|b\|^2\) with
    \(w=\talg(U(Z),Z)\), we get
    \[
        F_P(\talg(U(Z),Z))-F_P^*
        \leq
        \beta\|\talg(U(Z),Z)-\tilde w_Z\|^2 +
        \beta\|\hat w_Z-w_P^*\|^2 +
        \beta\|\tilde w_Z - \hat w_Z\|^2.
    \]
    In particular, \Cref{prop:core-swap-erm-distance} and \Cref{lem:erm-distance} together yield
    \[
        \E_{\talg, Z\sim P^N}
        \left[
            F_P(\talg(U(Z),Z))-F_P^*
        \right]
        =
        O\left(
        \kappa\frac{L^2}{\mu}
        \left(
            \frac1n+
            \left(\frac mn\right)^2e^{-2\eps/(d+2)}
        \right)
        \right),
    \]
    where we used \(\beta/\mu=\kappa\). This is the claimed SCO bound.

    Similarly, for mean estimation over \(\bB_d\), \(w_P^*=\mu_P:=\E_P[z]\) and
    $F_P(w)-F_P^* = \frac12\|w-\mu_P\|^2$.
    and therefore
    \[
        F_P(\talg(U(Z),Z))-F_P^*
        \lesssim
        \|\talg(U(Z),Z)-\tilde w_Z\|^2 +
        \|\hat w_Z-\mu_P\|^2 +
        \| \tilde w_Z - \hat w_Z \|^2
    \]
    Moreover, \(\hat w_Z=\frac1N\sum_{i=1}^N z_i\), so
    $\E[\|\hat w_Z-\mu_P\|^2] \le \frac1N \lesssim \frac1n$ and, 
    combining with \Cref{prop:core-swap-erm-distance}, we get the desired excess unlearning risk
    \[
        O\left(
            \frac1n+
            \left(\frac mn\right)^2e^{-2\eps/(d+2)}
        \right).
    \]
\end{proof}

\section{Lower Bounds for $(\eps, \delta)$-Unlearning}
\label{sec:lb}

In this section we prove a lower bound nearly matching the rate of our ERM unlearner. Surprisingly, we show that the optimal statistical rate cannot be improved by taking $\delta > 0$. This is in sharp contrast to differential privacy.
We begin with a mean estimation lower bound.

\begin{theorem}
    \label{thm:lb}
    Let $\talg(U, Z)$ be an $(\eps, \delta)$-unlearning algorithm with $\delta \leq 1/5$ and suppose that for all distributions $P$ on $\bB_d$ with mean $\mu_P$, all unlearning requests $U(Z) \subseteq Z$ with $|Z \setminus U(Z)| \geq n - m$, and any $n - m \leq N \leq n$, we have the mean squared error guarantee
    \begin{align*}
        \E_{\talg, Z \sim P^N}\left[ \lVert \talg(U(Z), Z) - \mu_P \rVert^2 \right] \leq \alpha.
    \end{align*}
    Then, it must be the case that
    \begin{align*}
        \alpha = \Omega\left(\frac{1}{n} + \left(\frac{m}{n} \right)^2 e^{-2\eps/(d+2)} \right).
    \end{align*}
\end{theorem}

The first term $\Omega(1/n)$ is the mean squared error lower bound even without unlearning requirements (see, e.g. \cite{duchi2021lecture}). The second term is what we will prove in this section. 
As a consequence of our mean estimation lower bound, we obtain our SCO lower bound:

\begin{corollary}
    Suppose the loss satisfies \Cref{ass:smooth-sc-sco} and that there is an $(\eps, \delta)$-unlearning algorithm $\talg$ with $\delta \leq 1/5$ and excess unlearning risk $\alpha$. Then we must have
    \begin{align*}
        \alpha \geq \Omega\left(\frac{L^2}{\mu}\left( \frac{1}{n} + \left( \frac{m}{n} \right)^2e^{-2\eps/(d+2)}\right)\right)
    \end{align*}
\end{corollary}

The first term $\Omega(L^2/(\mu n))$ holds for SCO without unlearning constraints~\cite{ny83}. The second follows from \Cref{thm:lb} and a standard reduction from SCO to mean estimation (see e.g., \cite[Proof of Theorem 8]{lowy2025private}).

\paragraph{Lower-bound intuition.}
We construct distributions whose means point in many separated directions, but whose nonzero signal samples can all be removed with high probability by a valid unlearning request. After this deletion, all hard instances induce the same retained dataset, so the unlearning guarantee forces their output distributions to be close to a common reference distribution. The packing size of the possible mean directions then limits how accurately all means can be recovered.

\textbf{Proof of \Cref{thm:lb}.} We now develop the tools that will be needed to prove~\Cref{thm:lb}. 
Just as our algorithm exploits covering geometry, our lower bound will exploit the packing geometry of the ball $\bB_d$.

\begin{lemma}
    \label{lemma:packing}
    For any $0 < \tau < 1$, there exists $\cV \subseteq \bB_d \setminus \frac{1}{2} \bB_d$ of size
    $|\cV| \geq \frac{1}{2}\left( \frac{1}{\tau} \right)^d$
    for which $\lVert v - v'\rVert > \tau$ for any $v \neq v' \in \cV$.
\end{lemma}

Indeed, by a standard volumetric packing argument (see e.g. Section 4.2 of \cite{vershynin2018high}), we can find a $\tau$-separated subset $\cV \subseteq \bB_d \setminus \frac{1}{2} \bB_d$ of size $|\cV| \geq \frac{\Vol(\bB_d \setminus \frac{1}{2} \bB_d)}{\Vol(\tau \bB_d)} = \frac{\Vol(\bB_d) - 2^{-d}\Vol(\bB_d)}{\tau^d \Vol(\bB_d)} \geq \frac{1}{2}\left( \frac{1}{\tau} \right)^d$.

Next, we show that, for a well-structured class of contaminated mixture distributions, deleting on the order of $m$ samples leaves a common distribution. Consequently, any algorithm that handles unlearning requests of size up to $m$ cannot effectively distinguish these distributions.

\begin{lemma}
    \label{lemma:contaminated_deletion}
    Let $n \geq 1$, $v \in \bB_d$, and $\eta \in [0,1/5]$ be such that \(m:=5\eta n\) is an integer. There exists a distribution $P_{v, \eta}$ on $\bB_d$ with mean $\eta v$ such that, given $Z \in \bB_d^n$, we can construct an unlearning request $U(Z) \subseteq Z$ of size $m$ for which
    \begin{align*}
        \mP_{Z \sim P_{v, \eta}^n}(Z \setminus U(Z) \neq \textbf{0}_{n - m}) \leq 2^{-m},
    \end{align*}
    where $\textbf{0}_{n - m} := (0, \dots, 0)$ denotes the dataset consisting of $n - m$ zeroes.
\end{lemma}

\begin{proof}
    Consider the contaminated mixture
    \begin{align*}
        P_{v, \eta} := (1 - \eta) 1_0 + \eta 1_v.
    \end{align*}
    Sample $Z \sim P_{v, \eta}^n$ and let $K \sim \textrm{Bin}(n, \eta)$ be the number of non-zero entries in $Z$. Now, consider the unlearning request $U(Z)$ such that, when $K > m$, $U(Z)$ removes any $m$ entries from $Z$ and, when $K \leq m$, $U(Z)$ removes all of the non-zero entries as well as other entries arbitrarily so that exactly $m$ entries are removed in total.

    Clearly, $Z \setminus U(Z) = \textbf{0}_{n - m}$ as long as $K \leq m$, so by a multiplicative Chernoff bound we get that
    \begin{align*}
        \mP_{Z \sim P_{v, \eta}^n}(Z \setminus U(Z) \neq \textbf{0}_{n - m})
            \leq \mP_{K \sim \textrm{Bin}(n, \eta)}(K > 5\eta n)
            \leq \left( \frac{e^4}{5^5} \right)^{\eta n}
            \leq (2^{-5})^{\eta n}
            = 2^{-m}.
    \end{align*}
\end{proof}

Finally, we require a slight variant of the standard technique of packing lower bounds from the differential privacy literature.

\begin{lemma}
    \label{lemma:squared_packing_method}
    Let $P^*, P_1, \dots, P_k$ be $\R^d$-valued distributions such that we can find disjoint events $E_1, \dots, E_k$ for which $P_i(E_i) \geq p$ as well as $P_i \approx_{\eps, \delta} P^*$ for each $i \leq k$. Assume additionally that we can find $\mu \in \R^d$ so that $\E_{W \sim P^*}[\lVert W - \mu \rVert^2] \leq \alpha$ and $\lVert w - \mu \rVert \geq \eta$ for each $w \in E_i$ and $i \leq k$. Then
    \begin{align*}
        e^\eps \geq \frac{k\eta^2}{\alpha} (p - \delta).
    \end{align*}
\end{lemma}

\begin{proof}
    By disjointness of the $E_i$, we have
    \begin{align*}
        \alpha
            \geq \E_{W \sim P^*}[\lVert W - \mu \rVert^2]
            \geq \sum_{i = 1}^k \int_{E_i} \lVert w - \mu \rVert^2 \, dP^*(w)
            \geq \eta^2 \sum_{i = 1}^k P^*(E_i)
            \geq k\eta^2 e^{-\eps}(p - \delta),
    \end{align*}
    which immediately yields the desired bound.
\end{proof}

We now have assembled all of the tools needed to show \Cref{thm:lb}. For convenience, we will write $g(X)|_{X \sim P}$ to denote the distribution of $g(X), X \sim P$.

\begin{proof}[Proof of \Cref{thm:lb}]
    Set $\eta := \frac{m}{5n}$. If $\sqrt{\alpha} \geq \frac{1}{8}\eta$, we are done, so assume that $\sqrt{\alpha} < \frac{1}{8}\eta$ and set $\tau := \frac{4\sqrt{\alpha}}{\eta} < 1$. By \Cref{lemma:packing}, we can find a $\tau$-separated $\cV \subseteq \bB_d \setminus \frac{1}{2} \bB_d$ of size
    $|\cV| \geq \frac{1}{2}\left( \frac{\eta}{4\sqrt\alpha} \right)^d$.

    Now, for each $v \in \cV$, recall the hard distribution $P_{v, \eta}$ and corresponding delete request $U_v(Z)$ as in \Cref{lemma:contaminated_deletion}. We claim that $\talg(U_v(Z), Z)|_{Z \sim P_{v, \eta}^n} \approx_{\eps, \delta + 2^{-m}} \talg(\varnothing, \textbf{0}_{n - m})$. Indeed, for any event $E$,
    \begin{align*}
        \mP_{\talg, Z \sim P_{v, \eta}^n}(\talg(U_v(Z), Z) \in E)
            & \leq \mP_{\talg, Z \sim P_{v, \eta}^n}(Z \setminus U_v(Z) = \textbf{0}_{n - m} \text{ and } \talg(U_v(Z), Z) \in E) + 2^{-m} \\
            & \leq \E_{Z \sim P_{v, \eta}^n}\left[ \mP_{\talg}(\talg(U_v(Z), Z) \in E) \,\middle|\, Z \setminus U_v(Z) = \textbf{0}_{n - m} \right] + 2^{-m} \\
            & \leq \E_{Z \sim P_{v, \eta}^n}\left[ e^{\eps} \mP_{\talg}(\talg(\varnothing, Z \setminus U_v(Z)) \in E) + \delta \,\middle|\, Z \setminus U_v(Z) = \textbf{0}_{n - m} \right] + 2^{-m} \\
            & = e^{\eps} \mP_{\talg}(\talg(\varnothing, \textbf{0}_{n - m}) \in E) + \delta + 2^{-m}
    \end{align*}
    and, analogously, $\mP_{\talg}(\talg(\varnothing, \textbf{0}_{n - m}) \in E) \leq e^\eps \mP_{\talg, Z \sim P_{v, \eta}^n}(\talg(U_v(Z), Z) \in E) + \delta + 2^{-m}$ as well.

    On the other hand, by Markov's inequality and our assumption on the mean squared error, for every $v \in \cV$ we have $\mP_{\talg, Z \sim P_{v, \eta}^n}(\talg(U_v(Z), Z) \in E_v) \geq 3/4$ where
    $E_v := \{ w : \lVert w - \eta v \rVert \leq 2\sqrt{\alpha} \}$.
    Since the packing $\cV$ has separation $\tau = \frac{4\sqrt{\alpha}}{\eta}$, the $E_v$ are disjoint. In addition, we have
    \begin{align*}
        \E_{\talg}\left[ \lVert \talg(\varnothing, \textbf{0}_{n - m}) \rVert^2 \right]
            = \E_{\talg, Z \sim 1_0^{n - m}}\left[ \lVert \talg(\varnothing, Z) \rVert^2 \right]
            \leq \alpha
    \end{align*}
    by the mean squared error assumption. Finally, for all $v \in \cV$ and $w \in E_v$, since $\lVert v \rVert > 1/2$, we must have
    \begin{align*}
        \lVert w \rVert
            \geq \lVert \eta v \rVert - \lVert w - \eta v \rVert
            \geq \eta/2 - 2\sqrt{\alpha}
            \geq \eta/4.
    \end{align*}
    Altogether, \Cref{lemma:squared_packing_method} yields
    \begin{align*}
        e^\eps
            \geq \frac{|\cV| (\eta/4)^2}{\alpha}\left( \frac{3}{4} - \delta - 2^{-m} \right)
            \geq \frac{\eta^2}{640 \alpha}\left( \frac{\eta}{4\sqrt\alpha} \right)^d
            \geq \left( \frac{\eta}{640\sqrt\alpha} \right)^{d + 2}
            = \left( \frac{m}{3200n\sqrt\alpha} \right)^{d + 2}.
    \end{align*}
    In particular, we have
    $\alpha \gtrsim \left( \frac{m}{n} \right)^2 e^{-2\eps/(d + 2)}$.
\end{proof}

\section{Conclusion}

We determined the minimax optimal rate for $(\eps, \delta)$-approximate machine unlearning in smooth strongly convex stochastic optimization up to a condition-number factor. Our results show that the optimal excess population risk consists of the usual statistical error plus an unlearning penalty that depends exponentially on the ratio \(\varepsilon/d\). In particular, retraining from scratch is statistically optimal when $\eps \lesssim d$, whereas in the regime $\eps \gg d$, our novel algorithm improves exponentially over retraining from scratch. For mean estimation over the unit ball, our upper and lower bounds match up to constants in the exponent, giving an essentially sharp characterization of the statistical price of unlearning in this canonical setting.

Several questions remain open for future work. First, our upper and lower bounds for general smooth strongly convex SCO differ by a condition-number factor; closing this gap would give a fully sharp minimax characterization beyond mean estimation and would help pave the way for nonsmooth and non-strongly convex SCO algorithms. Extending the minimax theory of unlearning to nonconvex optimization is an important goal to aim for. Finally, we leave open the optimal rates for unlearning under explicit storage and/or computational constraints. In particular, it is an important question whether the statistically optimal rate can be achieved by an algorithm that does not require a full pass over the dataset.

\begin{ack}
    We thank Hilal Asi for early discussions as well as Jacob Imola for discussions on efficient sampling.
    MR was supported by an NSERC CGS-D scholarship. GK was supported by a Canada CIFAR AI Chair, an NSERC Discovery Grant, and an Ontario Early Researcher Award.
\end{ack}

\clearpage
\bibliographystyle{alpha}
\bibliography{references}

\appendix

\section*{Appendix}

\section{Additional Related Work}
\label{app:related-work}

We discuss additional related work on machine unlearning, certified removal, and differentially private baselines. See the introduction for the works most directly comparable to our results.

\paragraph{Certified removal and approximate data unlearning.}
Guo et al.~\cite{guo2020certified} introduced certified data removal, requiring that a model after data removal be statistically indistinguishable from one trained without the removed data, and developed certified-removal mechanisms for linear classifiers. Izzo et al.~\cite{izzo2021approximate} proposed approximate unlearning methods for linear and logistic models whose unlearning-time cost is linear in the feature dimension and independent of the number of training samples. These works emphasize efficient approximate unlearning for specific model classes, while our work characterizes minimax population-risk rates for \(\varepsilon\)-unlearning in smooth strongly convex stochastic optimization.

\paragraph{Practical machine unlearning and exact unlearning frameworks.}
Bourtoule et al.~\cite{bourtoule2021machine} introduced SISA training, a practical framework for accelerating unlearning by sharding, isolating, slicing, and aggregating the training procedure. Their work helped popularize machine unlearning as a practical data-governance problem and focused primarily on reducing unlearning-time computation relative to retraining. This line of work is complementary to ours: we focus on the information-theoretic statistical cost of approximate unlearning in stochastic optimization.

\paragraph{Gradient-based and stability-based unlearning.}
Neel et al.~\cite{neel2021descent} study data unlearning for convex models and introduce gradient-based unlearning algorithms that can handle long sequences of adversarial updates with per-unlearning runtime and steady-state error not growing with the sequence length. Ullah et al.~\cite{ullah2021machine} connect unlearning to total-variation stability and design noisy-SGD-based algorithms with efficient unlearning procedures. Chien et al.~\cite{chien2024certified} later study certified machine unlearning via projected noisy stochastic gradient descent and establish approximate unlearning guarantees under convexity assumptions. These works are algorithmic and computationally motivated; whereas our focus is the minimax statistical rate.

\paragraph{Certified unlearning for convex and strongly convex learning.}
Sekhari et al.~\cite{sekhari2021remember} initiated a population-risk study of certified unlearning for convex learning and gave algorithms with unlearning-capacity guarantees, including a Newton-step method for smooth strongly convex losses. Their method improves over DP baselines in some regimes, but requires stronger smoothness assumptions such as Lipschitz Hessians. In Appendix~\ref{app:saks-improved}, we sharpen the analysis of this algorithm. Youssef et al.~\cite{youseffutility} study utility guarantees for in-distribution unlearning and propose a warm-start ERM procedure. In Appendix~\ref{app:youseff-corrected}, we identify a gap in their population-risk argument and provide a corrected retraining-level guarantee. Van Waerebeke et al.~\cite{van2025forget} study computational aspects of unlearning for strongly convex losses. These works are the closest algorithmic precursors to ours.

\paragraph{DP-based and restricted unlearning.}
Differential privacy gives a generic route to unlearning: by group privacy, an algorithm private enough for changes of size \(m\) automatically satisfies an unlearning guarantee for unlearning sets of size \(m\). Huang and Canonne~\cite{huang2023tight} give tight bounds for machine unlearning via differential privacy and for restricted unlearning models, including settings with a DP-based algorithm that does not use side information or the unlearning set. Their results show that DP-based unlearning is optimal in these restricted models. Our results show that unrestricted unlearning is more powerful than DP-based unlearning: by using the actual unlearning set and side information, retrain-from-scratch achieves superior excess risk to DP when $\eps \le d$ and our approximate unlearning algorithm can achieve an exponentially smaller unlearning penalty when \(\varepsilon\gg d\).

\paragraph{Other optimization settings.}
Liu et al.~\cite{liu2023certified} study certified unlearning for minimax models and derive generalization rates and unlearning-capacity bounds for convex-concave and strongly convex-strongly concave settings. Zou et al.~\cite{zou2025certified} study certified machine unlearning in proportional high-dimensional regimes, analyzing Newton-style procedures when the model dimension and sample size grow together. These works address different optimization or asymptotic settings and are complementary to our minimax analysis for smooth strongly convex stochastic optimization.

\paragraph{Broader unlearning literature.}
There is also a large empirical and systems-oriented literature on machine unlearning for neural networks, graphs, federated learning, recommendation systems, and foundation models. These works address important practical settings and evaluation questions, but typically do not provide excess risk bounds under certified \(\varepsilon\)-unlearning. We therefore focus our theoretical comparisons on certified unlearning methods and DP-based baselines with formal guarantees.

\section{Matching Bounds for Unlearning Under the Weaker Utility Assumption}
\label{app:post_unlearning}

In this section, we consider the effect of replacing the expected excess unlearning risk measure in \Cref{def:unlearning_utility} with the weaker definition seen in prior work, which we call expected excess \textit{post-unlearning} risk.

\begin{definition}
    \label{def:post_unlearning_utility}
    We say that $\talg$ achieves expected excess post-unlearning risk $\alpha$ if for every distribution $P \in \Delta(\ZZ)$ and any adversarial unlearning request $U(Z) \subseteq Z$ of size $|U(Z)| \leq m$, we have
    \[
        \E_{\talg, Z \sim P^n}\left[ F_P(\talg(U(Z), Z)) - F_P^* \right] \leq \alpha.
    \]
\end{definition}

\subsection{Upper Bound for Post-Unlearning}

\begin{theorem}[Post-Unlearning Upper Bound]
    \label{thm:core-diffusion-upper}
    Grant \Cref{ass:smooth-sc-sco} and assume $d \geq 2$. Then the unlearner \Cref{alg:core-diffusion} satisfies
    \(\varepsilon\)-unlearning and requires $O(d)$ time for randomization as well as $\tilde O(\kappa \cdot n)$ gradient and projection queries. Moreover, \Cref{alg:core-diffusion} achieves expected excess post-unlearning risk 
    \[
        \alpha
            =
            O\left(
            \kappa\frac{L^2}{\mu}
            \left(
                \frac1n+
                \left(\frac mn\right)^2e^{-2\eps/d}
            \right)
            \right).
    \]
    In the special case of mean estimation over the unit ball, i.e.
    \(\WW=\ZZ=\bB_d\) and \(f(w,z)=\frac12\|w-z\|^2\), \Cref{alg:core-diffusion} achieves expected excess post-unlearning risk
    \[
        \alpha = O\left(
            \frac1n+
            \left(\frac mn\right)^2e^{-2\eps/d}
        \right).
    \]
\end{theorem}

Critically, we note that \Cref{alg:core-diffusion} inspects the sizes of the dataset and the unlearning request to determine whether it is performing a dry or a true unlearning run. Although this algorithm achieves the optimal rate for the post-unlearning utility model, we suggest that this algorithm's behaviour goes against the spirit of machine unlearning. Therefore we advocate replacing the post-unlearning utility model \Cref{def:post_unlearning_utility} with the stronger utility model \Cref{def:unlearning_utility}.

\begin{figure}[htbp]
    \centering
    \includegraphics[width=0.85\linewidth]{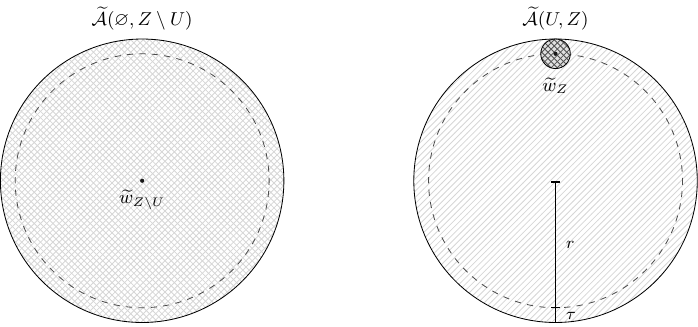}
    \caption{We compare the distribution of $\talg(\varnothing, Z \setminus U)$ to $\talg(U, Z)$ for the unlearner \Cref{alg:core-diffusion}. Lightly shaded regions are sampled uniformly with low probability and heavily shaded regions are sampled uniformly with high probability. In this case, $\talg(\varnothing, Z \setminus U)$ returns a noisy solution centered around the RFS ERM solution $\hat{w}_{Z \setminus U}$. $\talg(U, Z)$ has a similar distribution except some probability mass is transported into a small ball around the full ERM solution $\hat{w}_Z$.}
    \label{fig:core_diffusion_unlearner}
\end{figure}

\begin{algorithm}[t]
    \caption{Core-diffusion \(\varepsilon\)-unlearning for ERM}
    \label{alg:core-diffusion}
    \begin{algorithmic}[1]
        \REQUIRE Dataset \(T\); deletion set \(U\subseteq T\); parameters \(L,\mu,n,m,d,\varepsilon\)
        \STATE Set \(S\gets T\setminus U\), \(\rho \gets \frac{L}{\mu \sqrt{n}}\), \(r\gets \frac{2Lm}{\mu n} + 2\rho\), \(\tau\gets r\min\{1,2e^{-\varepsilon/d}\}\), \(K \gets \left(\frac{\tau + r}{\tau}\right)^d\), and \(\eta\gets \max\{e^{-\eps}, \frac{K - e^\eps}{K - 1} \} \)
        \STATE Compute deterministic $\tilde w_T$ and $\tilde w_S$ such that $\lVert \tilde w_T - \hat w_T \rVert, \lVert \tilde w_S - \hat w_S \rVert \leq \rho$

        \IF{\(U=\varnothing\) and \(|T|<n\)}
            \STATE Sample $w \sim \Unif((\tau + r)\bB_d + \tilde w_S)$
        \ELSE
            \STATE Sample $w \sim
                \begin{cases}
                    \Unif((\tau + r)\bB_d + \tilde w_S) & \text{w.p. } \eta \\
                    \Unif(\tau\bB_d + \tilde w_T) & \text{w.p. } 1 - \eta
                \end{cases}
            $
        \ENDIF
        \RETURN $\argmin_{w' \in \WW} \lVert w' - w \rVert$
    \end{algorithmic}
\end{algorithm}

\begin{proposition}
    \label{prop:core-diffusion-unlearning}
    The unlearner \Cref{alg:core-diffusion}, satisfies
    \(\varepsilon\)-unlearning.
\end{proposition}

\begin{proof}
    Unlearning is preserved by postprocessing, so we may exclude the final projection onto $\WW$ from our analysis.
    Now, fix $U \subseteq Z \in \ZZ^n$, $|U| \leq m$. If $U = \varnothing$, then clearly $\talg(U, Z) \equiv \talg(\varnothing, Z \setminus U)$, so assume $U \neq \varnothing$. In this case, the distribution of $\talg(\varnothing, Z \setminus U)$ before projection is
    \[
        W_{\varnothing, Z \setminus U} = \Unif(\underbrace{(\tau + r) \bB_d + \tilde{w}_{Z \setminus U}}_{=: B})
    \]
    whereas the distribution of running $\talg(U, Z)$ before projection is 
    \[
        W_{U, Z} = (1 - \eta) \Unif(\underbrace{\tau \bB_d + \tilde{w}_Z}_{=: G'}) + \eta \Unif(\underbrace{(\tau + r) \bB_d + \tilde{w}_{Z \setminus U}}_{= B}).
    \]

    We just need to argue that the likelihood ratio between these distributions lies in $[e^{-\eps}, e^\eps]$. Indeed, by the standard stability bound for strongly convex ERM (c.f. \cite[Lemma 6]{sekhari2021remember}),
    \[
        \lVert \tilde w_Z - \tilde w_{Z \setminus U} \rVert
            \leq \lVert \hat w_Z - \hat w_{Z \setminus U} \rVert + 2\rho
            \leq \frac{2Lm}{\mu n} + 2\rho
            = r.
    \]
    In particular, $G'$ is contained in $B$ and hence the likelihood ratio between these distributions differs only for $w$ contained in either $G'$ or $B$, exclusively. In the first case, if $w \in G'$, then
    \begin{align*}
        \frac{d \tilde W_{\varnothing, Z \setminus U}}{d \tilde W_{U, Z}}(w)
            = \frac{1/\Vol(B)}{(1 - \eta)/\Vol(G) + \eta/\Vol(B)}
            = \frac{1}{(1 - \eta)K + \eta}
            = \frac{1}{K - \eta(K - 1)}
            \in [e^{-\eps}, 1]
    \end{align*}
    since $\frac{K - e^\eps}{K - 1} \leq \eta \leq 1$. In the latter case, if $w \in B \setminus G'$, we have
    \begin{align*}
        \frac{d \tilde W_{\varnothing, Z \setminus U}}{d \tilde W_{U, Z}}(w)
            = \frac{1/\Vol(B)}{\eta/\Vol(B)}
            = \frac{1}{\eta}
            \in [1, e^\eps]
    \end{align*}
    since $e^{-\eps} \leq \eta \leq 1$. In particular, $\tilde W_{\varnothing, Z \setminus U}$ and $\tilde W_{U, Z}$ are $(\eps, 0)$-indistinguishable, as desired.
\end{proof}

\begin{proposition}
    \label{prop:core-diffusion-erm-distance}
    Assume $d \geq 2$. For every dataset \(Z\in\cZ^n\) and every deletion set
    \(U\subseteq Z\) with \(|U|\le m\), the output $\talg(U,Z)$ of \Cref{alg:core-diffusion} satisfies
    \[
        \E_{\talg}
        \left[
            \left\|
            \talg(U,Z)-\tilde w_Z
            \right\|^2
        \right]
        =
        O\left(
            \frac{L^2}{\mu^2}
            \left(\frac mn\right)^2
            e^{-2\eps/d} +
            \frac{L^2}{\mu^2 n}
        \right).
    \]
\end{proposition}

\begin{proof}
    Recall as in the proof of \Cref{prop:core-diffusion-unlearning} that $\talg(U, Z)$ is exactly the projection of
    \[
        w \sim W = (1 - \eta) \Unif(\tau \bB_d + \tilde{w}_Z) + \eta \Unif((\tau + r) \bB_d + \tilde{w}_{Z \setminus U})
    \]
    onto the convex domain $\WW$. But $\tilde{w}_Z \in \WW$, so  
    \begin{align*}
        \E_{\talg}[\| \talg(U, Z) - \tilde w_Z \|^2]
            \leq \E_{w \sim W}[\| w - \tilde w_Z \|^2]
            \leq (1 - \eta) \tau^2 + \eta \max_{w \in (\tau + r) \bB_d + \tilde{w}_{Z \setminus U}} \| w - \tilde w_Z \|^2.
    \end{align*}

    The first term is bounded by
    \[
        \tau^2
            = \left(r \min\left\{1, 2e^{-\eps/d}\right\}\right)^2
            \lesssim \frac{L^2}{\mu^2} \left( \frac{m}{n} \right)^2 e^{-2\eps/d} +
            \frac{L^2}{\mu^2 n}.
    \]

    As for the second term, recalling that $\|\tilde w_{Z \setminus U} - \tilde w_Z\| \leq r$ as well as our choice of $\tau \leq r$, we have
    \[
        \|w - \tilde w_Z\|^2
            \lesssim \|w - \tilde w_{Z \setminus U}\|^2 + \|\tilde w_{Z \setminus U} - \tilde w_Z\|^2
            \leq (\tau + r)^2 + r^2
            \lesssim r^2
            \lesssim \frac{L^2}{\mu^2} \left( \frac{m}{n} \right)^2 + \frac{L^2}{\mu^2 n}
    \]
    for any \(w\in(\tau + r) \bB_d + \hat{w}_{Z \setminus U}\). Moreover, if $2e^{-\eps/d} > 1$, then $\eta \leq 1 = 1^2 \lesssim e^{-2\eps/d}$. More importantly, if $2e^{-\eps/d} \leq 1$, then $\tau = 2re^{-\eps/d}$ by choice of $\tau$, in which case
    \begin{align*}
        K
            = \left( \frac{\tau + r}{\tau} \right)^d
            \leq \left( \frac{2r}{\tau} \right)^d
            = \left( \frac{2r}{2re^{-\eps/d}} \right)^d
            = e^\eps,
    \end{align*}
    and hence
    \begin{align*}
        \eta
            = \max\left\{ e^{-\eps}, \frac{K - e^\eps}{K - 1} \right\}
            = e^{-\eps}
            \leq e^{-2\eps/d}
    \end{align*}
    in this case as well. The result now follows by combining these bounds.
\end{proof}

As in the proof of \Cref{thm:core-swap-upper}, \Cref{thm:core-diffusion-upper} now follows by combining
\Cref{prop:core-diffusion-unlearning,prop:core-diffusion-erm-distance} with \Cref{lem:erm-distance,lem:erm-approximation}.

\subsection{Lower Bounds for $(\eps, \delta)$-Post-Unlearning}

In this section we prove a lower bound matching the rate of \Cref{alg:core-diffusion}.
At its core, this reduces to showing a mean estimation lower bound.

\begin{theorem}
    \label{thm:post_unlearning_lb}
    Let $\talg(U, Z)$ be an $(\eps, \delta)$-unlearning algorithm with $\delta \leq 1/7$ and suppose that, for all distributions $P$ on $\bB_d$ with mean $\mu_P$ and unlearning requests $U(Z) \subseteq Z$ with $|U(Z)| \leq m$, we have the mean squared error guarantee
    \begin{align*}
        \mP_{\talg, Z \sim P^n}\left( \lVert \talg(U(Z), Z) - \mu_P \rVert^2 \leq \alpha \right) \geq \frac{2}{3}.
    \end{align*}
    Then, it must be the case that
    \begin{align*}
        \alpha = \Omega\left(\frac{1}{n} + \left(\frac{m}{n} \right)^2 e^{-2\eps/d} \right).
    \end{align*}
\end{theorem}

By Markov's inequality, the same lower bound applies up to constants for any algorithm with \textit{expected} mean squared error $O(\alpha)$. Moreover, as in \Cref{sec:lb} the sampling error lower bound $\Omega(1/n)$ holds even without unlearning requirements, so it just remains to derive the second term. As in \Cref{sec:lb}, this implies the following SCO lower bound:

\begin{corollary}
    Suppose the loss satisfies \Cref{ass:smooth-sc-sco} and that there is an $(\eps, \delta)$-unlearning algorithm $\talg$ with $\delta \leq 1/7$ and excess post-unlearning risk $\alpha$. Then we must have
    \begin{align*}
        \alpha \geq \Omega\left(\frac{L^2}{\mu}\left( \frac{1}{n} + \left( \frac{m}{n} \right)^2e^{-2\eps/d}\right)\right)
    \end{align*}
\end{corollary}

The argument for \Cref{thm:post_unlearning_lb} is very similar to that of \Cref{sec:lb}, except that we rely on a weaker form of the packing technique.

\begin{lemma}
    \label{lemma:packing_bound}
    Let $P^*, P_1, \dots, P_k$ be distributions over a common space such that we can find disjoint events $E_1, \dots, E_k$ for which $P_i(E_i) \geq p$ as well as $P_i \approx_{\eps, \delta} P^*$ for each $i \in \{1, \dots, k\}$. Then
    \begin{align*}
        e^\eps \geq k(p - \delta).
    \end{align*}
\end{lemma}

\begin{proof}
    Indeed, by disjointness, we have
    \begin{align*}
        1 \geq \sum_{i = 1}^k P^*(E_i)
            \geq \sum_{i = 1}^k e^{-\eps} (P_i(E_i) - \delta)
            \geq \sum_{i = 1}^k e^{-\eps} (p - \delta)
            = ke^{-\eps}(p - \delta).
    \end{align*}
\end{proof}

\begin{proof}[Proof of \Cref{thm:post_unlearning_lb}]
    Set $\eta := \frac{m}{5n}$. If $\sqrt{\alpha} \geq \frac{1}{2}\eta$, we are done, so assume that $\sqrt{\alpha} < \frac{1}{2}\eta$ and set $\tau := \frac{2\sqrt{\alpha}}{\eta} < 1$. By \Cref{lemma:packing}, we can find a $\tau$-separated $\cV \subseteq \bB_d$ of size
    \begin{align*}
        |\cV| \geq \frac{1}{2}\left( \frac{\eta}{2\sqrt\alpha} \right)^d.
    \end{align*}

    Now, for each $v \in \cV$, recall the hard distribution $P_{v, \eta}$ and corresponding delete request $U_v(Z)$ as in \Cref{lemma:contaminated_deletion}.
    As in the proof of \Cref{thm:lb}, we have that $\talg(U_v(Z), Z)|_{Z \sim P_{v, \eta}^n} \approx_{\eps, \delta + 2^{-m}} \talg(\varnothing, \textbf{0}_{n - m})$ because $\talg$ is $(\eps, \delta)$-unlearning.

    On the other hand, by our assumption on the mean estimation error, for every $v \in \cV$ we have
    \begin{align*}
        \frac{2}{3}
            \leq \mP_{\talg, Z \sim P_{v,\eta}^n}\left( \lVert \talg(U_v(Z), Z) - \eta v \rVert \leq \sqrt\alpha \right).
    \end{align*}
    As the $v \in \cV$ are $\frac{2\sqrt{\alpha}}{\eta}$-separated, the $\{ w : \lVert w - \eta v \rVert \leq \sqrt{\alpha} \}$ are disjoint and hence \Cref{lemma:packing_bound} yields
    \begin{align*}
        e^\eps
            \geq |\cV|\left( \frac{2}{3} - \delta - 2^{-m} \right)
            \geq \frac{1}{42}\left( \frac{\eta}{2\sqrt\alpha} \right)^d
            \geq \left( \frac{\eta}{84\sqrt\alpha} \right)^d
            = \left( \frac{m}{420n\sqrt\alpha} \right)^d.
    \end{align*}
    In particular, we have
    $\alpha \gtrsim \left( \frac{m}{n} \right)^2 e^{-2\eps/d}$.
\end{proof}

\section{Retrain-from-scratch Upper Bound}
\label{app:rfs}

We first record the performance of the most basic exact-unlearning algorithm:
retraining from scratch. Given a unlearning request \(U\subseteq Z\), define
\[
    \talg_{\mathrm{RFS}}(U,Z)
    :=
    \hat w_{Z\setminus U}
    \in
    \argmin_{w\in\WW}\widehat F_{Z\setminus U}(w),
\]
where
\[
    \widehat F_{Z\setminus U}(w)
    :=
    \frac{1}{|Z\setminus U|}
    \sum_{z_i\in Z\setminus U} f(w,z_i).
\]
This algorithm ignores the original model and recomputes an empirical risk minimizer on the retained data. Its main drawback is that it requires storing the full dataset.

\begin{theorem}[Retrain-from-scratch upper bound]
    \label{thm:rfs}
    Let $f$ satisfy \Cref{ass:smooth-sc-sco} and let \(\kappa=\beta/\mu\). Then retraining from scratch is exact \(0\)-unlearning. Moreover,
    \[
        \E_{Z\sim P^n}
        \left[
        \max_{U\subseteq Z:\, |U|\le m}
        \left(
        F_P(\hat w_{Z\setminus U})-F_P^*
        \right)
        \right]
        \lesssim
        \kappa\frac{L^2}{\mu}
        \left(
            \frac{1}{n}
            +
            \left(\frac{m}{n}\right)^2
        \right).
    \]
\end{theorem}

\begin{proof}
    The exact-unlearning guarantee is immediate. After receiving \(U\), the algorithm outputs \(\hat w_{Z\setminus U}\), which is exactly the output obtained by running the learning algorithm directly on the retained dataset \(Z\setminus U\). Hence the two output distributions are identical.

    It remains to prove the excess-risk bound. Let
    \[
        \hat w_Z\in\argmin_{w\in\WW}\widehat F_Z(w),
        \qquad
        w^*\in\argmin_{w\in\WW}F_P(w).
    \]
    By \(\beta\)-smoothness and the stationarity assumption, we have
    \[
        F_P(w)-F_P^*
        \le
        \frac{\beta}{2}\|w-w^*\|^2
        \qquad
        \forall w\in\WW.
    \]
    Therefore,
    \begin{align}
        &\E
        \left[
        \max_{U\subseteq Z:\, |U|\le m}
        \left(
        F_P(\hat w_{Z\setminus U})-F_P^*
        \right)
        \right]
        \nonumber\\
        &\qquad\le
        \frac{\beta}{2}
        \E
        \left[
        \max_{U\subseteq Z:\, |U|\le m}
        \|\hat w_{Z\setminus U}-w^*\|^2
        \right]
        \nonumber\\
        &\qquad\le
        \beta
        \E\|\hat w_Z-w^*\|^2
        +
        \beta
        \E
        \left[
        \max_{U\subseteq Z:\, |U|\le m}
        \|\hat w_{Z\setminus U}-\hat w_Z\|^2
        \right],
        \label{eq:rfs-basic-decomposition}
    \end{align}
    where the last step uses \(\|a+b\|^2\le 2\|a\|^2+2\|b\|^2\).

    We bound the two terms in~\eqref{eq:rfs-basic-decomposition}. First, standard stability/generalization bounds for \(L\)-Lipschitz, \(\mu\)-strongly convex ERM imply
    \[
        \E\!\left[F_P(\hat w_Z)-F_P^*\right]
        \lesssim
        \frac{L^2}{\mu n}.
    \]
    By \(\mu\)-strong convexity of \(F\),
    \[
        \frac{\mu}{2}\E\|\hat w_Z-w^*\|^2
        \le
        \E\!\left[F_P(\hat w_Z)-F_P^*\right],
    \]
    and hence
    \begin{equation}
        \E\|\hat w_Z-w^*\|^2
        \lesssim
        \frac{L^2}{\mu^2 n}.
        \label{eq:rfs-stat-term}
    \end{equation}

    Second, the unlearning stability bound for strongly convex ERM gives, uniformly over all \(U\subseteq Z\) with \(|U|\le m\),
    \[
        \|\hat w_{Z\setminus U}-\hat w_Z\|
        \lesssim
        \frac{Lm}{\mu n}.
    \]
    For example, this is precisely the unlearning stability estimate used in~\cite[Lemma 6]{sekhari2021remember}. Therefore,
    \begin{equation}
        \E
        \left[
        \max_{U\subseteq Z:\, |U|\le m}
        \|\hat w_{Z\setminus U}-\hat w_Z\|^2
        \right]
        \lesssim
        \frac{L^2m^2}{\mu^2n^2}.
        \label{eq:rfs-unlearning-term}
    \end{equation}

    Combining~\eqref{eq:rfs-basic-decomposition}, \eqref{eq:rfs-stat-term}, and~\eqref{eq:rfs-unlearning-term} yields
    \[
        \E
        \left[
        \max_{U\subseteq Z:\, |U|\le m}
        \left(
        F_P(\hat w_{Z\setminus U})-F_P^*
        \right)
        \right]
        \lesssim
        \beta\frac{L^2}{\mu^2}
        \left(
            \frac{1}{n}
            +
            \left(\frac{m}{n}\right)^2
        \right).
    \]
    Since \(\beta/\mu=\kappa\), this is
    \[
        \lesssim
        \kappa\frac{L^2}{\mu}
        \left(
            \frac{1}{n}
            +
            \left(\frac{m}{n}\right)^2
        \right),
    \]
    as claimed.
\end{proof}

\begin{remark}
    The proof uses only Lipschitzness, strong convexity, smoothness on the constrained domain, as well as stationarity of the minimizer. Under these assumptions, the condition-number factor multiplies both the usual statistical term and the unlearning term. Removing the factor \(\kappa\) from the \(1/n\) term would require an additional condition such as a smooth self-bounding inequality.
\end{remark}

\section{Corrected analysis of the ERM warm-start algorithm of
\texorpdfstring{\cite{youseffutility}}{Youssef et al.}}
\label{app:youseff-corrected}

In this section we revisit the ERM-based warm-start unlearning algorithm of~\cite{youseffutility}. At a high level, their algorithm is very close to retraining from scratch: after receiving a unlearning request, it approximately minimizes the empirical risk on the retained dataset, using the original trained model as a warm start, and then adds noise to certify unlearning. We give a corrected utility analysis showing that this approach achieves a population-risk bound matching retraining from scratch, up to the optimization error and the noise variance\footnote{We note that the authors have also corrected the analysis in a subsequent preprint~\cite{youseffutility_arxiv_v3}.}.

We use our notation rather than the notation of~\cite{youseffutility}. Let \(Z=(z_1,\ldots,z_n)\sim P^n\), and for a unlearning set \(U\subseteq Z\), let
\[
    \widehat F_{Z\setminus U}(w)
    :=
    \frac{1}{|Z\setminus U|}
    \sum_{z_i\in Z\setminus U} f(w,z_i),
    \qquad
    \hat w_{Z\setminus U}
    \in
    \argmin_{w\in\WW}\widehat F_{Z\setminus U}(w).
\]

\paragraph{The algorithm.}
The learner first computes an empirical risk minimizer \(\hat w_Z\) on the full dataset and stores the data \(Z\). Upon receiving a unlearning request \(U\), the unlearning algorithm runs an optimization method, initialized at \(\hat w_Z\), on the retained empirical objective \(\widehat F_{Z\setminus U}\). Let \(w_U\) denote the resulting approximate retained-data ERM. Finally, the algorithm outputs
\begin{equation}
    \label{eq:app-youseff-output}
    \tilde w_U
    :=
    \Pi_{\WW}(w_U+v),
\end{equation}
where \(v\) is the noise used to certify unlearning. When \(w_U=\hat w_{Z\setminus U}\) and \(v=0\), this is exactly retraining from scratch.

The theorem below isolates the utility guarantee of this template. It applies to the warm-start algorithm of~\cite{youseffutility} once their optimization accuracy and noise calibration are substituted.

\begin{theorem}[Corrected utility bound for ERM warm-start unlearning]
    \label{thm:youseff-corrected}
    Let $f$ satisfy \Cref{ass:smooth-sc-sco} and let \(\kappa=\beta/\mu\). Suppose that for every unlearning set \(U\subseteq Z\) with \(|U|\le m\), the unlearning-time optimizer returns \(w_U\in\WW\) satisfying
    \begin{equation}
    \label{eq:app-youseff-empirical-accuracy}
        \widehat F_{Z\setminus U}(w_U)
        -
        \widehat F_{Z\setminus U}(\hat w_{Z\setminus U})
        \le \gamma .
    \end{equation}
    Assume also that the noise \(v\) in~\eqref{eq:app-youseff-output} satisfies
    \[
        \E\|v\|^2 \le \sigma^2 .
    \]
    Then
    \[
        \E_{Z\sim P^n}
        \left[
        \sup_{U\subseteq Z:\,|U|\le m}
        \E\!\left[
            F_P(\tilde w_U)-F_P^*
            \mid Z,U
        \right]
        \right]
        \lesssim
        \kappa\frac{L^2}{\mu}
        \left(
            \frac1n+\left(\frac{m}{n}\right)^2
        \right)
        +
        \frac{\beta}{\mu}\gamma
        +
        \beta\sigma^2 .
    \]
    In particular, if
    \[
        \gamma
        \lesssim
        \frac{L^2}{\mu}
        \left(
            \frac1n+\left(\frac{m}{n}\right)^2
        \right)
        \qquad\text{and}\qquad
        \sigma^2
        \lesssim
        \frac{L^2}{\mu^2}
        \left(
            \frac1n+\left(\frac{m}{n}\right)^2
        \right),
    \]
    then the algorithm achieves the retraining from scratch rate
    \[
        \E_{Z\sim P^n}
        \left[
        \sup_{U\subseteq Z:\,|U|\le m}
        \E\!\left[
            F_P(\tilde w_U)-F_P^*
            \mid Z,U
        \right]
        \right]
        \lesssim
        \kappa\frac{L^2}{\mu}
        \left(
            \frac1n+\left(\frac{m}{n}\right)^2
        \right).
    \]
\end{theorem}

\begin{proof}
    Fix \(Z\) and \(U\subseteq Z\) with \(|U|\le m\). Let
    \[
        w^*\in\argmin_{w\in\WW}F_P(w).
    \]
    By \(\beta\)-smoothness and the stationarity assumption,
    \begin{equation}
    \label{eq:app-youseff-smooth-upper}
        F_P(w)-F_P^*
        \le
        \frac{\beta}{2}\|w-w^*\|^2
        \qquad
        \forall w\in\WW.
    \end{equation}
    Using nonexpansiveness of projection and the inequality
    \(\|a_1+\cdots+a_4\|^2\le 4\sum_{j=1}^4\|a_j\|^2\), we obtain
    \begin{align}
        \E\!\left[F_P(\tilde w_U)-F_P^*\mid Z,U\right]
        &\le
        \frac{\beta}{2}
        \E\!\left[
            \|\tilde w_U-w^*\|^2
            \mid Z,U
        \right]
        \nonumber\\
        &\lesssim
        \beta\|\hat w_Z-w^*\|^2
        +
        \beta\|\hat w_{Z\setminus U}-\hat w_Z\|^2
        \nonumber\\
        &\qquad
        +
        \beta\|w_U-\hat w_{Z\setminus U}\|^2
        +
        \beta\E\|v\|^2 .
        \label{eq:app-youseff-decomposition}
    \end{align}

    We now bound each term. First, as in the proof of \Cref{thm:rfs}, standard stability/generalization bounds for \(L\)-Lipschitz, \(\mu\)-strongly convex ERM imply
    \[
        \E_Z\|\hat w_Z-w^*\|^2
        \lesssim
        \frac{L^2}{\mu^2 n}.
    \]
    Second, the unlearning stability bound for strongly convex ERM gives, uniformly over all \(U\subseteq Z\) with \(|U|\le m\),
    \[
        \|\hat w_{Z\setminus U}-\hat w_Z\|
        \lesssim
        \frac{Lm}{\mu n}.
    \]
    Third, by \(\mu\)-strong convexity of \(\widehat F_{Z\setminus U}\) and the empirical accuracy condition~\eqref{eq:app-youseff-empirical-accuracy},
    \[
        \frac{\mu}{2}
        \|w_U-\hat w_{Z\setminus U}\|^2
        \le
        \widehat F_{Z\setminus U}(w_U)
        -
        \widehat F_{Z\setminus U}(\hat w_{Z\setminus U})
        \le
        \gamma,
    \]
    and hence
    \[
        \|w_U-\hat w_{Z\setminus U}\|^2
        \le
        \frac{2\gamma}{\mu}.
    \]
    Finally, \(\E\|v\|^2\le\sigma^2\) by assumption.

    Substituting these bounds into~\eqref{eq:app-youseff-decomposition}, taking the supremum over \(U\), and then taking expectation over \(Z\), gives
    \[
        \E_Z
        \left[
        \sup_{U\subseteq Z:\,|U|\le m}
        \E\!\left[
            F_P(\tilde w_U)-F_P^*
            \mid Z,U
        \right]
        \right]
        \lesssim
        \beta\frac{L^2}{\mu^2}
        \left(
            \frac1n+\left(\frac{m}{n}\right)^2
        \right)
        +
        \frac{\beta}{\mu}\gamma
        +
        \beta\sigma^2 .
    \]
    Since \(\beta/\mu=\kappa\), the claimed bound follows.
\end{proof}

\paragraph{Where the prior analysis breaks.}
The proof of~\cite[Proposition 1]{youseffutility} attempts to convert an empirical-risk guarantee on the retained sample into a population-risk guarantee. The key error occurs in the line after the authors invoke smoothness of the loss. In our notation, smoothness can only give
\[
    F_P(w)-F_P(\theta_S^*)
    \le
    \langle \nabla F_P(\theta_S^*), w-\theta_S^*\rangle
    +
    \frac{\beta}{2}\|w-\theta_S^*\|^2 ,
\]
where \(\theta_S^*\) is the empirical minimizer used in their argument. The proof then effectively treats the linear term
\[
    \langle \nabla F_P(\theta_S^*), w-\theta_S^*\rangle
\]
as zero. This is not justified: \(\theta_S^*\) minimizes the empirical risk, not the population risk, so in general
\[
    \nabla F_P(\theta_S^*)\neq 0.
\]
Moreover, in constrained optimization, even the population minimizer need not have zero gradient. Thus the omitted first-order term can dominate the claimed bound, and the proposition does not establish the population-risk guarantee stated in~\cite{youseffutility}.

The corrected analysis above avoids this step. Instead of expanding the population risk around an empirical minimizer and dropping the first-order term, we compare the warm-start output to the exact retained-data ERM, use strong convexity to convert empirical optimization error into parameter error, and then apply the same stability argument as retraining from scratch. This yields a retraining-level rate, plus the explicit contributions of optimization error and unlearning noise.

\section{A Sharper Analysis of the Newton-step Algorithm of
\texorpdfstring{\cite{sekhari2021remember}}{Sekhari et al.}}
\label{app:saks-improved}

We revisit the Newton-step unlearning algorithm of~\cite{sekhari2021remember}. Their algorithm was introduced as a way to improve over generic differentially private baselines for smooth strongly convex losses. In this section, we show that the same algorithm admits a sharper population-risk analysis than the one originally given.

Throughout this section, assume that for every \(z\in\ZZ\), the loss \(f(\cdot,z)\) is \(L\)-Lipschitz, \(\mu\)-strongly convex, and \(\beta\)-smooth over \(\WW\). We write \(\kappa=\beta/\mu\). We additionally assume, as in~\cite{sekhari2021remember}, that the Hessian is \(M\)-Lipschitz:
\begin{equation}
    \label{eq:app-saks-hessian-lipschitz}
    \|\nabla^2 f(w,z)-\nabla^2 f(w',z)\|
    \le M\|w-w'\|
    \qquad
    \forall w,w'\in\WW,\ z\in\ZZ .
\end{equation}
For a dataset \(Z=(z_1,\ldots,z_n)\), write
\[
    \widehat F_Z(w)
    :=
    \frac1n\sum_{i=1}^n f(w,z_i),
    \qquad
    \hat w_Z\in\argmin_{w\in\WW}\widehat F_Z(w).
\]

The algorithm of~\cite{sekhari2021remember} initializes at the full-data ERM \(\hat w_Z\). After receiving a unlearning set \(U\subseteq Z\), it takes one Newton step with respect to the retained empirical objective. Equivalently, the deterministic part of the update can be written as
\begin{equation}
    \label{eq:app-saks-newton-step}
    \bar w_U
    :=
    \hat w_Z
    -
    \nabla^2 \widehat F_{Z\setminus U}(\hat w_Z)^{-1}
    \nabla \widehat F_{Z\setminus U}(\hat w_Z),
\end{equation}
where
\[
    \widehat F_{Z\setminus U}(w)
    :=
    \frac{1}{|Z\setminus U|}
    \sum_{z_i\in Z\setminus U} f(w,z_i).
\]
The unlearning algorithm outputs a noisy version of this update,
\begin{equation}
    \label{eq:app-saks-noisy-output}
    \tilde w_U
    :=
    \Pi_{\WW}(\bar w_U+v),
\end{equation}
where \(v\) is calibrated Gaussian noise for \((\varepsilon,\delta)\)-unlearning, or calibrated multivariate Laplace noise for pure \(\varepsilon\)-unlearning. The projection is post-processing and can only improve the distance-to-\(w^*\) bounds used below.

The original analysis of~\cite{sekhari2021remember} gives, up to constants,
\[
    \frac{L^2}{\mu}\frac{m}{n}
    +
    \frac{L^3 M}{\mu^3}
    \frac{m^2\sqrt{d\log(1/\delta)}}{\varepsilon n^2}
\]
for the Gaussian-noise \((\varepsilon,\delta)\)-unlearning version when $m \le n/2$. When $\delta = 0$, substituting Laplace noise for Gaussian in their algorithm yields the same bound but with $\sqrt{d \log(1/\delta)}$ replaced by $d$. We show that a more direct population-risk argument improves the generalization term from \(m/n\) to \(1/n+(m/(n-m))^2\), which is \(1/n+(m/n)^2\) when \(m\le n/2\), and squares the privacy contribution when converting parameter error to excess risk.

\begin{theorem}[Improved analysis of the Newton-step algorithm]
    \label{thm:saks-improved}
    Let $f$ satisfy \Cref{ass:smooth-sc-sco} and additionally assume that \(f(\cdot,z)\) has \(M\)-Lipschitz Hessian for every \(z\in\ZZ\). Let \(\tilde w_U\) be the Newton-step unlearning algorithm of~\cite{sekhari2021remember} for unlearning sets of size at most \(m<n\).

    For the Gaussian-noise version calibrated to satisfy \((\varepsilon,\delta)\)-unlearning,
    \[
        \sup_{U\subseteq Z:\,|U|\le m}
        \E\!\left[F_P(\tilde w_U)-F_P^*\right]
        \lesssim
        \kappa\frac{L^2}{\mu}
        \left(
            \frac1n+\left(\frac{m}{n-m}\right)^2
        \right)
        +
        \beta
        \left(\frac{M L^2}{\mu^3}\right)^2
        \frac{d\,m^4\log(1/\delta)}{\varepsilon^2 n^4}.
    \]
    For the pure \(\varepsilon\)-unlearning version obtained by replacing Gaussian noise with multivariate Laplace noise,
    \[
        \sup_{U\subseteq Z:\,|U|\le m}
        \E\!\left[F_P(\tilde w_U)-F_P^*\right]
        \lesssim
        \kappa\frac{L^2}{\mu}
        \left(
            \frac1n+\left(\frac{m}{n-m}\right)^2
        \right)
        +
        \beta
        \left(\frac{M L^2}{\mu^3}\right)^2
        \frac{d^2m^4}{\varepsilon^2 n^4}.
    \]
    In particular, if \(m\le n/2\), then the deterministic part of both bounds simplifies to
    \[
        \kappa\frac{L^2}{\mu}
        \left(
            \frac1n+\left(\frac{m}{n}\right)^2
        \right).
    \]
\end{theorem}

\begin{proof}
    Fix any unlearning set \(U\subseteq Z\) with \(|U|\le m\). The bounds below are uniform over such \(U\), so taking the supremum gives the theorem.

    Let \(w^*\in\argmin_{w\in\WW}F_P(w)\). By \(\beta\)-smoothness and the stationarity assumption,
    \begin{equation}
    \label{eq:app-saks-smooth-upper}
        F_P(w)-F_P^*
        \le
        \frac{\beta}{2}\|w-w^*\|^2
        \qquad
        \forall w\in\WW .
    \end{equation}
    Using nonexpansiveness of projection and \(\|a+b\|^2\le 2\|a\|^2+2\|b\|^2\), we obtain
    \begin{align}
        \E\!\left[F_P(\tilde w_U)-F_P^*\right]
        &\le
        \frac{\beta}{2}\E\|\tilde w_U-w^*\|^2 \nonumber\\
        &\le
        \beta\E\|\bar w_U-\hat w_Z\|^2
        +
        \beta\E\|\hat w_Z-w^*\|^2
        +
        \beta\E\|v\|^2 .
        \label{eq:app-saks-risk-decomposition}
    \end{align}

    We now bound the three terms in~\eqref{eq:app-saks-risk-decomposition}. First, by the definition of the Newton step~\eqref{eq:app-saks-newton-step},
    \[
        \bar w_U-\hat w_Z
        =
        -
        \nabla^2 \widehat F_{Z\setminus U}(\hat w_Z)^{-1}
        \nabla \widehat F_{Z\setminus U}(\hat w_Z).
    \]
    The proof of~\cite[Lemma 9]{sekhari2021remember} gives, uniformly over all \(U\subseteq Z\) with \(|U|\le m\),
    \begin{equation}
    \label{eq:app-saks-newton-displacement}
        \|\bar w_U-\hat w_Z\|
        \lesssim
        \frac{Lm}{\mu(n-m)}.
    \end{equation}
    Consequently,
    \begin{equation}
    \label{eq:app-saks-deterministic-term}
        \E\|\bar w_U-\hat w_Z\|^2
        \lesssim
        \frac{L^2m^2}{\mu^2(n-m)^2}.
    \end{equation}

    Second, standard stability/generalization bounds for \(L\)-Lipschitz, \(\mu\)-strongly convex ERM imply
    \begin{equation}
    \label{eq:app-saks-erm-excess}
        \E\!\left[F_P(\hat w_Z)-F_P^*\right]
        \lesssim
        \frac{L^2}{\mu n}.
    \end{equation}
    By \(\mu\)-strong convexity of \(F\),
    \[
        \frac{\mu}{2}\E\|\hat w_Z-w^*\|^2
        \le
        \E\!\left[F_P(\hat w_Z)-F_P^*\right].
    \]
    Combining this with~\eqref{eq:app-saks-erm-excess} gives
    \begin{equation}
    \label{eq:app-saks-statistical-term}
        \E\|\hat w_Z-w^*\|^2
        \lesssim
        \frac{L^2}{\mu^2 n}.
    \end{equation}

    Third, \cite{sekhari2021remember} show that the deterministic Newton-step map has unlearning sensitivity
    \begin{equation}
    \label{eq:app-saks-unlearning-sensitivity}
        \Delta
        \lesssim
        \frac{M L^2}{\mu^3}\frac{m^2}{n^2}.
    \end{equation}
    Thus, for the Gaussian mechanism calibrated to \((\varepsilon,\delta)\)-unlearning,
    \begin{equation}
    \label{eq:app-saks-gaussian-noise}
        \E\|v\|^2
        \lesssim
        \frac{d\,\Delta^2\log(1/\delta)}{\varepsilon^2}
        \lesssim
        \left(\frac{M L^2}{\mu^3}\right)^2
        \frac{d\,m^4\log(1/\delta)}{\varepsilon^2 n^4}.
    \end{equation}
    For the pure \(\varepsilon\)-unlearning version based on multivariate Laplace noise,
    \begin{equation}
    \label{eq:app-saks-laplace-noise}
        \E\|v\|^2
        \lesssim
        \frac{d^2\Delta^2}{\varepsilon^2}
        \lesssim
        \left(\frac{M L^2}{\mu^3}\right)^2
        \frac{d^2m^4}{\varepsilon^2 n^4}.
    \end{equation}

    Plugging~\eqref{eq:app-saks-deterministic-term}, \eqref{eq:app-saks-statistical-term}, and either~\eqref{eq:app-saks-gaussian-noise} or~\eqref{eq:app-saks-laplace-noise} into~\eqref{eq:app-saks-risk-decomposition} gives
    \[
        \E\!\left[F_P(\tilde w_U)-F_P^*\right]
        \lesssim
        \beta\frac{L^2}{\mu^2}
        \left(
            \frac1n+\left(\frac{m}{n-m}\right)^2
        \right)
        +
        \beta\E\|v\|^2.
    \]
    Since \(\beta L^2/\mu^2=\kappa L^2/\mu\), the claimed bounds follow. If \(m\le n/2\), then \((n-m)^{-1}\le 2n^{-1}\), giving the simplified deterministic term.
\end{proof}

\begin{remark}
    The improvement comes from analyzing population risk through the squared distance to the full-data ERM \(\hat w_Z\). The original analysis effectively pays a first-order generalization term of order \(m/n\). By instead using smoothness of \(F\), the deterministic displacement of the Newton update contributes quadratically, giving \((m/(n-m))^2\), while the full-data ERM contributes the usual \(1/n\) statistical term, up to the condition-number factor.
\end{remark}

\end{document}